\documentclass{svproc_KGT}

\usepackage{url}

\usepackage[utf8]{inputenc}
\usepackage{graphicx}
\usepackage{amsmath}
\usepackage{amssymb,xcolor}
\usepackage{mathtools}
\usepackage{cancel}
\usepackage{enumitem}
\usepackage[noend]{algorithmic}
\usepackage{algorithm}
\usepackage{wrapfig}
\usepackage{hyperref}
\usepackage{tikz-cd}
\usepackage{tikz}
\usetikzlibrary{automata, positioning}

\usepackage[scriptsize]{subfigure}
\usepackage{epsfig} 

\usepackage{color}

\DeclarePairedDelimiter{\ceil}{\lceil}{\rceil}

\newcommand{\Z}{\mathbb{Z}} 
\newcommand{\R}{\mathbb{R}} 
\newcommand{\N}{\mathbb{N}} 

\newcommand{\cat}{{}^\frown}

\newcommand{\Pgrid}{\mathcal{P}_\textrm{grid}}
\newcommand{\Pglobal}{\mathcal{P}_\textrm{global}}
\newcommand{\CCC}{\mathcal{C}^\circ(\R^2)}

\def\Nat{{\mathbb{N}}}
\def\Z{{\mathbb{Z}}}

\def\histu{{\tilde{u}}}

\def\P{{\cal P}}

\def\atan2{\operatorname{atan2}}

\usepackage{acronym}
\acrodef{Ispace}[I-space]{\emph{information space}}
\acrodef{Istate}[I-state]{\emph{information state}}
\acrodef{Imap}[I-map]{\emph{information mapping}}
\acrodef{ITS}[ITS]{\emph{information transition system}}
\acrodef{ITSs}[ITSs]{\emph{information transition systems}}
\acrodef{DITS}[DITS]{\emph{deterministic information transition system}}
\acrodef{NITS}[NITS]{\emph{nondeterministic information transition system}}
\acrodef{POMDPs}[POMDPs]{\emph{partially observable Markov decision processes}}
\acrodef{PSRs}[PSRs]{\emph{predictive state representations}}

\begin{document}
\mainmatter              
\title{Universal Plans:\\ One Action Sequence to Solve Them All!}
\titlerunning{Universal Plans}  
%
\author{Kalle G. Timperi, Alexander J. LaValle, Steven M. LaValle%
\thanks{This work was supported by a European Research Council (ERC AdG, ILLUSIVE: Foundations of Perception Engineering, 101020977) and Academy of Finland (PERCEPT 322637)  {\tt\small (e-mail: firstname.lastname@oulu.fi).}}%
}
\institute{Center for Ubiquitous Computing \\
Faculty of Information Technology and Electrical Engineering \\
University of Oulu, Finland}

\maketitle              

\begin{abstract}
This paper introduces the notion of a universal plan, which when executed, is guaranteed to solve all planning problems in a category, regardless of the obstacles, initial state, and goal set.  Such plans are specified as a deterministic sequence of actions that are blindly applied without any sensor feedback.  Thus, they can be considered as pure exploration in a reinforcement learning context, and we show that with basic memory requirements, they even yield optimal plans.  Building upon results in number theory and theory of automata, we provide universal plans both for discrete and continuous (motion) planning and prove their (semi)completeness.  The concepts are applied and illustrated through simulation studies, and several directions for future research are sketched.

\end{abstract}
%

\section{Introduction}

A planning algorithm typically takes as input a robot movement model, environment model (obstacles), initial state, and goal state (or states), and must compute a plan that would bring the robot from the initial to the goal while avoiding obstacles.  Suppose we are in a discrete-time, predictable setting, in which case a plan is a sequence of actions.  Consider varying the initial and goal states.  We naturally expect the planning algorithm to produce different action sequences for different planning instances, and report failure if it is impossible.

What would happen if we demand that a single action sequence must work for a set of several instances?  Clearly, a plan could easily fail by changing the initial state, see Figure~\ref{fig:finite_grid_plans}.  However, we could make plans more robust by assuming an action `does nothing' if the robot is about to step into an obstacle.   Under this setting, we raise the seemingly absurd question:  Can a single action sequence solve {\em all} planning instances?  In this paper, we introduce infinite action sequences, each of which solves all solvable planning instances under the assumption that only the movement model is fixed.  The obstacles, initial, and goal are allowed to vary.  We call such an action sequence that solves all problems a {\em universal plan}.  The existence of a universal plan implies that a planning algorithm could be trivialized by merely ignoring the input model and reporting the same action sequence every time! Even stranger, we provide conditions under which a universal plan discovers in finite time an optimal trajectory for any planning instance to which
it is applied.

\begin{figure}[tb]
\begin{tabular}{cc}
\includegraphics[width=6cm]{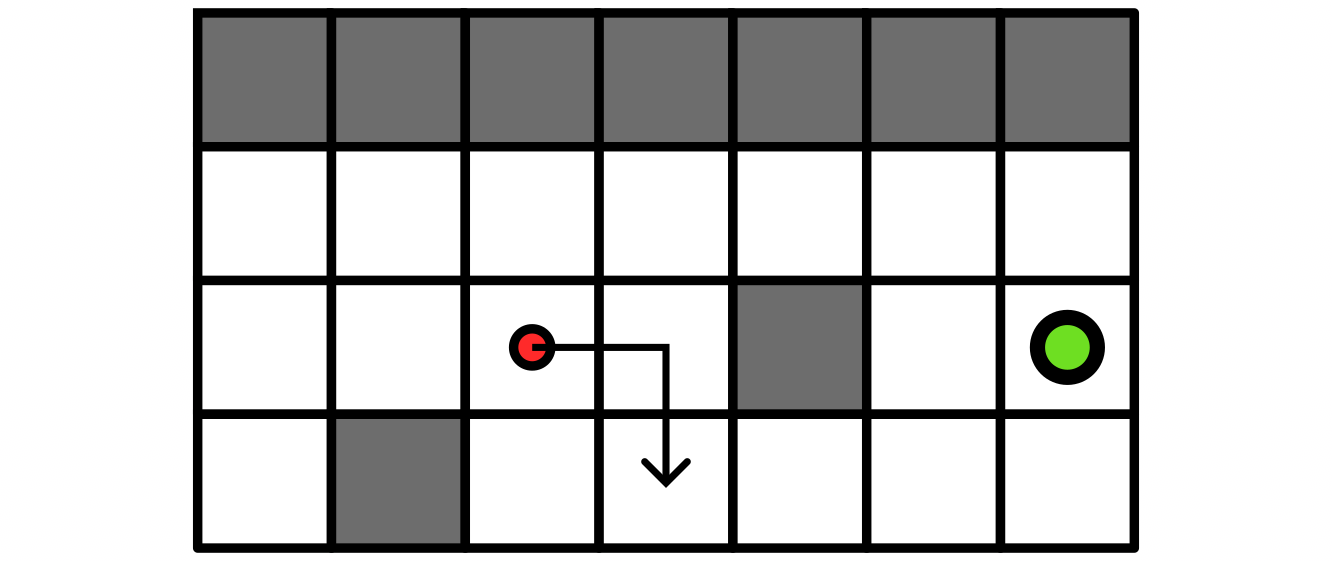} &
\includegraphics[width=6cm]{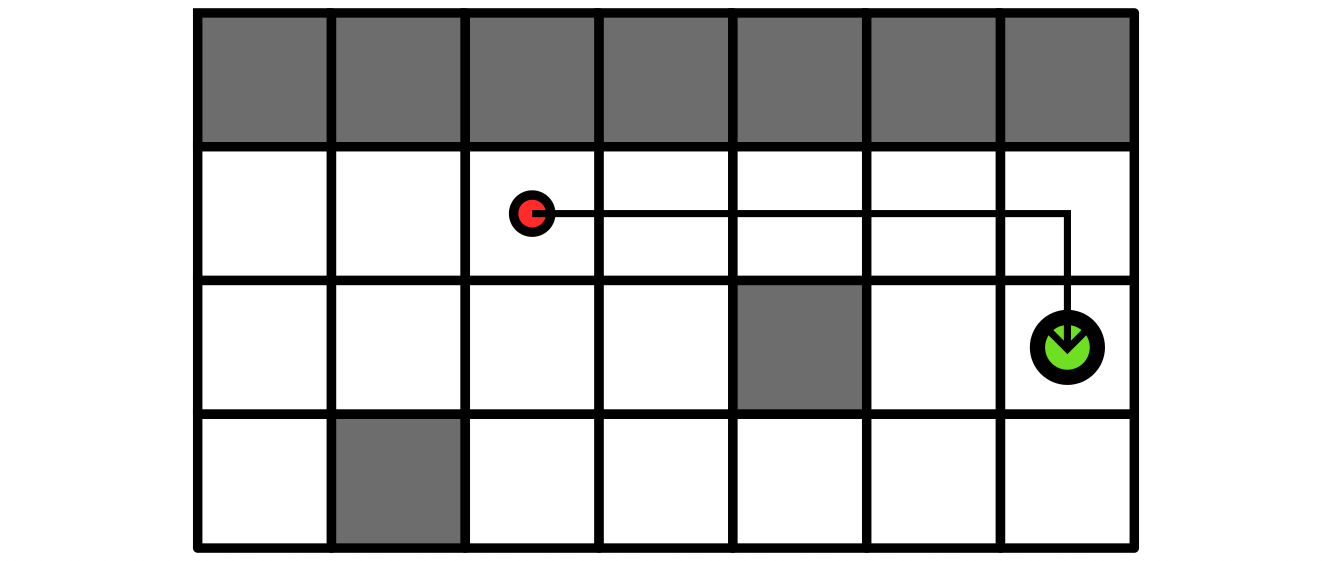} \\
a. & b.
\end{tabular}
\caption{\label{fig:finite_grid_plans} Applying a finite action sequence $(\rightarrow, \rightarrow, \rightarrow, \rightarrow, \downarrow)$ in a grid environment (white cells are free space, grey cells are obstacles, and \includegraphics[width=0.2cm]{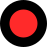} is the start state). (a) The robot tries to move four times to the right, but is kept still by the obstacle, then moves down. (b) The robot is not obstructed, and actually reaches the goal \includegraphics[width=0.3cm]{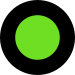}.
} 
\end{figure}

Random walks, and randomization in general, are related; however, we want deterministic guarantees that the problem will be solved.  This is significantly harder to ensure.  Thus, we want to explore the completeness of such a solution, without resorting to the weaker notion of probabilistic completeness.  Furthermore, unless an external physical process is used for entropy production \cite{HuLiaWonZho09,TurBarKelMcKBaiBoy18}, so-called `random' sequences are merely produced by deterministic pseudorandom number generators in practice.  Realizing this, why not consider other deterministic sequences, some of which may be more suitable for planning? 

Why should we care about universal plans?  At the very least, they improve our understanding of planning algorithms, especially the notions of completeness and asymptotic optimality, which have been argued for decades as important, desirable properties of a planning algorithm.  This study also helps to clarify the role of randomization in planning \cite{AleKarLipLovRac79,BarKavLatLiMotRag96,BarLat90,Erd92}; in light of our work, it seems that the {\em complexity} of a random sequence is more important than its probabilistic properties~\cite{Kal94}.  Thus, we can focus more on optimizing sequences that have sufficient complexity, rather than worry about their probabilistic properties or interpretations.  

Another reason to study universal plans is for the exploration component of ``exploration vs.~exploitation" for machine learning applied to robot systems.  Reinforcement learning requires a sufficiently rich collection of state-action-state triples.  A universal plan could be an ideal exploration strategy, which is ultimately optimized for a policy that achieves lower costs (higher rewards) via exploiting the parts already visited.

\section{Universal Plans}\label{sec:up}
The question of the existence of universal plans can be approached through the following two interrelated questions:
\begin{itemize}
    \item[{Q1:}] For an environment $X$ and a goal state $x_G \in X$, can we characterize the set of finite plans that take the robot from any initial state to the goal state?
    \item[{Q2:}] Is there a deterministic process that produces an infinite action sequence $\histu$ that is guaranteed to take the robot from any initial state $x_I$ in any environment $X$ to any goal state $x_G \in X$?
\end{itemize}

For Q1, asymptotic upper and lower bounds for the length of universal sequences have been established in the context of graph exploration~\cite{AleKarLipLovRac79,BarBorKarLinWer89,Bri87}. It is also meaningful to ask, given the set of all such plans for a given environment $X$, what other environments $X'$ share the same set of plans? In other words, in what sense do sets of plans characterize the environments for which they succeed? Our main focus is to address question Q2 by studying plans generated from rich or normal number sequences.

\subsection{Normal and Rich Sequences}
In his seminal 1909 paper~\cite{Bor09}, Borel defined \emph{normal numbers}~\cite{BaiCra02,Bug12,KuiNie74}, whose decimal expansion satisfies a certain uniformity property: a real numbers $x$ is \emph{normal} in base $b$, if every length-$k$ sequence of symbols $0, 1, \ldots, b-1$ occurs in the base-$b$ expansion of $x$ with the asymptotic frequency $b^{-k}$. We also say that such a decimal expansion is a \emph{normal sequence}. Every normal number is irrational, and it is known that Lebesgue-almost all reals are in fact normal~\cite{Bor09}. Furthermore, almost every `truly random' (Chaitin random) number is normal~\cite{Cal94,LiVit19}. However, despite a century of efforts, basic questions about normal numbers remain unanswered. It is not known, for example, whether household irrationals such as $\pi$, $\sqrt{2}$, and $\ln 2$ are normal in any base (they all are conjectured to be normal). An excellent survey on the computation of the digits of $\pi$ and considerations regarding its normality is provided in~\cite{BaiBor16}. The construction of concrete examples of normal numbers is likewise challenging~\cite{Cha33,CopErd46,Sto76}. However, there are algorithms for computing normal numbers in nearly linear time~\cite{LutMay21}.
The earliest example of a normal number is Champernowne's number, obtained by concatenating the decimal expressions of all natural numbers into an infinite sequence~\cite{Cha33}.

A related notion is that of a \emph{rich number}, for which the decimal expansion is a \emph{rich} (or \emph{disjunctive}) \emph{sequence}. This means that every finite word in the alphabet $\{0, 1, \ldots, b-1\}$ occurs in the decimal expansion at least once (and thus infinitely often)~\cite{Com83}. Hence, every normal sequence is rich, but the converse is not generally true. The study of the properties of rich and normal numbers connects such separate fields as number theory~\cite{BerDowVan22}, Kolmogorov complexity~\cite{LiVit19,Tak13}, automata theory~\cite{BecHei13} and mathematical logic~\cite{Com83,Her96}.

From the perspective of universality, richness
guarantees the sampling of all possible finite action sequences, when individual symbols/numerals are interpreted as robot actions. A pivotal aspect of this approach turns out to be the way in which richness or normality of a sequence is preserved when the sequence is further split into disjoint subsequences. This is a delicate problem in its own right, although many results are known~\cite{BerDowVan22,BerVan19,Kam73,Van17}. In our construction of a universal scale-free plan in Section~\ref{sec:scale} we rely on so-called prefix selection rules to guarantee the normality of subsequences. We also implemented  a dynamic (non-oblivious) selection rule, which outperformed the prefix selection rule in our numerical experiments. Proving its universality is a topic for future work. 


%
What exactly is the difference, then, between random sampling versus following a normal number sequence? In a randomly generated plan, a fixed, finite action sequence $\histu$ can appear at any given time with positive probability. In contrast, if the actions of the robot are obtained from the decimals of a rich or normal number, $\histu$ will only show up at specific, predetermined points in the decimal expansion. This calls for different, non-probabilistic proof techniques which, we argue, shed new light also on the existing probabilistic approaches.


\subsection{Basic concepts and notation}
Following notation in \cite{Lav06}, consider a discrete feasible planning problem expressed in terms of: 1) a nonempty {\em state space} $X$, 2) a nonempty {\em action space} $U(x)$ for every $x \in X$, 3) a {\em state transition function} $f$ that yields some $f(x,u) \in X$ for every $x \in X$ and $u \in U(x)$, 4) an {\em initial state} $x_I \in X$, and 5) a nonempty {\em goal set} $X_G \subset X$.  

A {\em plan} is a sequence of actions $\histu = (u_1,\ldots,u_K)$, which are designed to be applied in succession from stage $1$ to stage $K$.  Initially, $x_1 = x_I$.  Then, $u_1 \in U(x_1)$ is applied to obtain $x_2 = f(x_1,u_1)$. This process repeats until $x_{K+1} = f(x_K,u_K)$ is reached.  A {\em successful} plan is one in which $x_{K+1} \in X_G$. When indicating the final state after executing a sequence of actions $\histu$, we sometimes use the notation
$x \cat \histu := f(\cdots f(f(x, u_1), u_2) \cdots, u_K) = x_{K+1}$.

It will be convenient to consider plans that do not have a fixed length $K$.  Instead, an {\em infinite plan} $\histu = (u_1,u_2,\ldots)$ may be defined. This may be written as a mapping $\histu: \Nat \rightarrow U$, in which $\Nat = \{1,2,\ldots\}$ is the set of natural numbers and $U = \bigcup_{x \in X} U(x)$ is the union of all possible state-dependent actions.\footnote{Each action $u_k$ is chosen from the appropriate set $U(x)$, where $x$ is the state that is reached from $x_I$ by applying the action sequence $(u_1, \ldots, u_{k-1})$.} For an infinite plan to be successful, we merely require that there exists a stage $k \in \Nat$ for which $x_k \in X_G$.  During execution, the robot may have a {\em goal detector}, signalling if $x_k \in X_G$, thereby terminating the plan (rather than running forever).

Let $P = (X,U,f,x_I,X_G)$ refer to a {\em planning instance}, which would ordinarily correspond to the input to a planning algorithm.  Let $\P$ refer to a set of planning instances.  A plan $\histu$ is called {\em universal} with respect to a set $\P$ if it is a successful plan for all $P \in \P$.
Thus, our quest is to find fixed plans that are successful for the largest possible $\P$.

Consider some possibilities for $\P$.  Let $\P(X,U,f)$ be the set of all possible planning instances given $X$, $U$, and $f$.  This could correspond to a fixed set of obstacles, but any $x_I \in X$ and nonempty $X_G \subset X$ are allowed.  Let $\P(X,U)$ denote the set of all planning instances by allowing any valid $f$, $x_I$, and $X_G$.  Let $\P(X,U,F)$ denote the set of all planning instances generated by allowing any $f \in F$, in which $F$ is a set of state transition functions. Section \ref{sec:maze} will define $F$ to correspond to the set of all finite, 2D grid-based planning problems.

\section{Universal Plans for Grid Search Problems}\label{sec:maze}
Maze search is a classical problem in the intersection of robotics and theoretical computer science~\cite{BluKoz78,Bud77,Bud78,CohFraIlcKorPel08,LavBraLin04}. It can be viewed as a restricted case of the more general graph search problem~\cite{DikFraKraPel04,FraIlcPeePelPel05}, applied to finite, connected planar graphs with maximum degree four.
Two complementary interpretations of the problem have been explored in the literature, namely, (i) constructing a representation of an unknown graph environment from sensory feedback~\cite{RivSch93,RivSch94}, and (ii) identifying
properties of the robot-environment coupling that ensure a priori the complete exploration of the graph even in the case of minimal sensing~\cite{BluKoz78,CohFraIlcKorPel08}. We explore the limits of the latter 
viewpoint when sensing 
is removed completely.

\subsection{The Robot Grid Search Problem}
In this section, we define the discrete robot grid search planning problem and establish the existence of universal plans for such problems. 

\begin{definition}[\textbf{Connected Grid}] \label{Def_Connected_Grid}
A subset $Z \subset \Z \times \Z$ is \emph{connected}, if for any two points $z_1,z_2 \in Z$ there exists some $n$-tuple $(x_1, \ldots, x_n) \subset Z$ for which $x_1 = z_1$, $x_n = z_2$ and $\|x_{k+1} - x_k\| = 1$ for all $k = 1, \ldots, n-1$.
\end{definition}

We now formally define the robot grid search problem. The actions available to the robot, $\{\texttt{left, right, up, down}\}$ are identified with the corresponding
unit-length vectors $\{v_\textrm{l}, v_\textrm{r}, v_\textrm{u}, v_\textrm{d}\} := \{(-1,0), (1,0), (0,1), (0,-1)\}$.

\begin{definition}[\textbf{Robot grid search problem}] \label{Def_robot_grid_search_problem}
A \emph{robot grid search problem} is a planning problem $P_X = (X, U, f, x_I, x_G)$, in which $X \subset \mathbb{Z} \times \mathbb{Z}$ is a finite, connected environment, $U = \{v_\textrm{l}, v_\textrm{r}, v_\textrm{u}, v_\textrm{d}\}$, the initial and goal states are $x_I, x_G \in X$, respectively, and the transition function is given by
\begin{equation} \label{Eq_robot_movement_model}
f(x, u) = \begin{cases}
    x + u, & \textrm{if} \,\,\,  x + u \in X, \\
    x, & \textrm{if} \,\,\, x + u \notin X.
\end{cases}
\end{equation}
We denote by $\Pgrid (\Z \times \Z)$ the collection of all robot grid search problems.
\end{definition}
The interpretation of~\eqref{Eq_robot_movement_model} is that if the robot is trying to move through the environment boundary, it will stay put. Note that the robot is assumed to be sensorless -- it will not know whether it moved after trying to execute an action.

When using a rich number $\alpha = \alpha_1\alpha_2\ldots$ to generate action sequences, we need to convert the digits $\alpha_k$ to corresponding actions via a suitable mapping.

\begin{definition}[\textbf{Rich and normal plans}] \label{Def_rich_and_normal_plans}
Let $P \in \Pgrid (\Z \times \Z)$ be a robot grid search problem with actions space $U$, let $\alpha = \alpha_1\alpha_2\ldots$ be a rich/normal number in base $b$, and let $c: \mathcal{B} \to U$, where $\mathcal{B} = \{0, \ldots, b-1\}$. Then the infinite sequence $\big(c(\alpha_n)\big)_{n=1}^\infty$ is a \emph{rich/normal plan (based on $\alpha$)}.
\end{definition}

We begin by noting that, formally, our definition of a robot grid search problem coincides with that of a \emph{deterministic finite automaton (DFA)}.
\begin{definition}[\textbf{Deterministic Finite Automaton}] \label{Def_DFA}
A \emph{deterministic finite automaton (DFA)} is a $5$-tuple $(Q, \Sigma, \delta, q_0, F)$, where $Q$ is a finite set of states, $\Sigma$ is a finite set called \emph{alphabet}, $\delta$ is the \emph{transition} function, $q_0 \in Q$ is the start state, and $F \subset Q$ is the set of \emph{accept states}.
\end{definition}
By identifying the set of automaton states $Q$ with the environment (maze or graph) $X$, the set of inputs $\Sigma$ with the set $U$ of robot actions, and the set of accept states $F$ with the set $X_G$ of goal states, we may interpret search problems for finite mazes and graphs as DFA exploration problems.

We begin by introducing the concept of an \emph{essential class}. The definition is equivalent to~\cite[Definition 4.1]{Com83}.

\begin{definition}[\textbf{Reachable state; Essential class}]
Let $P_X$ be a planning problem and $x, x' \in X$. Then $x'$ is \emph{reachable} from state $x$ in $P_X$, if there is an action sequence $\histu$ such that $f(x,\histu) = x'$. A maximal set $S \subset X$ of states all reachable from each other in $X$ is called an \emph{essential class} of $P_X$.
\end{definition}

An essential class thus corresponds to a strongly connected planning problem (subautomaton) $P_S = (S, U, f)$, meaning that every state is reachable from every other state in $S$. Since each robot grid search problem $P_X$ is by definition strongly connected (the set $X$ is assumed to be a connected grid), the only essential class for the planning problem is the state space $S=X$ itself.

Generally, for any $P_X \in \Pgrid (\Z \times \Z)$ and any $x \in X$, there exists some action sequence $\histu$ for which the state $x\cat \histu$ belongs to some essential class~\cite[Proposition 4.2]{Com83}. 
From this we obtain the following result (\cite[Lemma 4.3]{Com83}):

\begin{proposition}[\textbf{Rich plans converge to essential classes (Compton)}] \label{Prop_rich_plans_go_to_essential}
Let $P_X = (X, U, f, x_I)$ be a planning problem with a finite state space $X$, and let $\big(c(\alpha_n)\big)_{n=1}^\infty$ be a rich plan corresponding to some rich number $\alpha$. Then there exists some essential class $S$ and $N \in \N$, for which $x_I \cat \big(c(\alpha_1), \ldots, c(\alpha_n) \big) \in S$ for all $n \geq N$, and the robot will visit every $x \in S$ infinitely many times.
\end{proposition}

\begin{lemma}[\textbf{Essential class of finite trajectories}] \label{Lemma_Essential_class_of_finite_trajectories}
Let $P_X = (X,U,f,x_I) \in \Pgrid (\Z \times \Z)$ and for each $k \in \N$ let $\widehat{X}^k := \big(X \times U^k, U, \widehat{f}, (x_I, (u_1, \ldots, u_k)) \big)$ be a DFA where $\widehat{f}\big((x,(u_1, \ldots, u_k)), u\big) := \big(f(x, u_1), (u_2, \ldots, u_k, u) \big)$ for all $x \in X$ and $u \in U$.
Then $X \times U^k$ is an essential class for $\widehat{X}^k$.
\end{lemma}

\begin{proof}
Let $\widehat{z} := \big(x, (u_1, \ldots, u_k)\big)$ and $\widehat{z}' := \big(x', (u_1', \ldots, u_k')\big)$.
Since every robot grid search problem is strongly connected, there exists some action sequence $\histu$ for which $x \cat (u_1, \ldots, u_k) \cat \histu = x'$. Define $\widehat{w} := \histu \cat (u_1', \ldots, u_k')$. Then
\begin{align*}
\widehat{f}(\widehat{z}, \widehat{w}) &=
\widehat{f}\big((x,(u_1, \ldots, u_k)), \histu \cat (u_1', \ldots, u_k')\big) \\
&= \big(x', (u_1', \ldots, u_k') \big) \\
&= \widehat{z}'.~\square
\end{align*}
\end{proof}

We immediately obtain the following result.

\begin{corollary}[\textbf{Rich plan tries every action sequence from every state}] \label{Cor_rich_tries_everything}
In every $P_X 
\in \Pgrid (\Z \times \Z)$ a robot following a rich plan $\histu$ will apply every finite action sequence from every state $x \in X$ infinitely many times.
\end{corollary}

\begin{proof}
Let $\widehat{X}^k$ be the DFA in Lemma~\ref{Lemma_Essential_class_of_finite_trajectories}. According to Proposition~\ref{Prop_rich_plans_go_to_essential}, the automaton $\widehat{X}^k$ will visit every state $\big(x, (u_1, \ldots, u_k)\big) \in X \times U^k$ infinitely many times. This is clearly equivalent with the robot visiting state $x \in X$ and applying the action sequence $(u_1, \ldots, u_k)$ immediately afterwards. $\square$
\end{proof}

Corollary~\ref{Cor_rich_tries_everything} implies in particular that there exists, for any robot grid search problem $\mathcal{P}_X$, some finite action sequence $\histu(X)$ for which the trajectory $x_I \cat \histu(X)$ visits every state $x \in X$ (in particular each $x \in X_G$) from any initial state $x_I$.

A strong form of a universal plan for a fixed robot grid search problem is that of a \emph{synchronizing sequence}~\cite{Hof23,Vol08}.\footnote{The related notion of \emph{homing sequences}~\cite{San04} concerns inferring the final state after applying an action sequence, on the basis of both actions and observations.}
An action sequence $\histu$ \emph{synchronizes} a planning problem $P_X \in \Pgrid (\Z \times \Z)$, if there exists some $x_F \in X$ for which $x_I \cat \histu = x_F$ for all initial states $x_I$.

We say that $P_X$ is \emph{synchronizing}, if some $\histu$ synchronizes it. In a synchronizing planning problem, uncertainty about the current state of the robot can in principle be reduced by applying a single action sequence, and the concept has found useful applications~\cite{Nat86,OkaLav06}.

A planning problem $P_X$ is \emph{pairwise synchronizing}, if for any $x, x' \in X$ there exists an action sequence $\histu$ for which $x \cat \histu = x' \cat \histu$.
It turns out that this seemingly weaker property actually implies synchronicity~\cite[Theorem 1.14]{San04}. Thus, Lemma~\ref{Lemma_Pairwise_Synch} below implies that every robot grid search problem is synchronizing.

\begin{lemma}[\textbf{Robot grid search problems are pairwise synchronizing}] \label{Lemma_Pairwise_Synch}
Let $(X, f, U) \in \Pgrid (\Z \times \Z)$. Then for all $x_I, x_I' \in X$, there exists a finite action sequence $\histu$ for which $x_I \cat \histu = x_I' \cat \histu$.
\end{lemma}

\begin{proof}
Let $x_I, x_I' \in X$. Since $X$ is connected, there exists some shortest action sequence $\tilde{u}_1$ for which $x_I\cat \histu_1 = x_I'$. Since $\histu_1$ is the shortest such sequence, the path from $x_I$ to $x_I'$ contains no collisions with the environment boundary. Define $x_1 := x_I \hspace{0.1mm}\cat \tilde{u}_1 = x_I'$ and $x_1' := x_I' \cat \tilde{u}_1$. Iterating, we similarly choose for each $n > 1$ some shortest action sequence $\histu_n$ for which $x_{n-1} ~\hspace{-1.5mm}\cat \histu_n = x_{n-1}'$, and define $x_n := x_{n-1} ~\hspace{-1.5mm}\cat \histu_n = x_{n-1}'$ and $x_n' := x_{n-1}' \cat \histu_n$. Intuitively, a robot starting from $x_I$ is chasing the one starting from $x_I'$. We claim that there exists some $N \in \N$ for which $x_N = x_N'$. This intuitively follows because only the escaping robot is hitting the environment boundaries.

Assume instead that there exists some $n \in \N$ for which $x_n \neq x_n'$ and that neither of the robots ever hits the environment boundary after this index. This implies that for every $k > n$ the displacements $x_k - x_{k-1}$ and $x_k' - x_{k-1}'$ are the same. Hence, $x_k - x_n = (k-n)(x_{n+1} - x_n)$ for all $k > n$, which is a contradiction since the distance $\|x_k - x_n\|$ is bounded from above due to the finiteness of $X$. Thus, if $x_n \neq x_n'$, there must exist some $k > n$ for which the escaping robot travels fewer steps from $x_{k-1}'$ to $x_k'$ than the one chasing it from $x_{k-1}$ to $x_k$. Since the next action sequence $\histu_{k+1}$ minimizes the actions necessary to reach $x_k'$ from $x_k$, it will have less actions than the previous sequence $\histu_k$. It follows from the above that $\|x_N' - x_N\| = 0$ for some $N\in\N$. Thus, $x_I \cat \histu_1 \cat \cdots \cat \histu_N = x_I' \cat \histu_1 \cat \cdots \cat \histu_N$. $\square$
\end{proof}

It should be stressed that none of the results presented in this section require a two-dimensional setting; they generalize to all finite dimensions. Rich numbers can thus provide universal plans for a wide array of discrete planning problems.

\section{Universal Plans for Continuous Environments} \label{Sec_Universal_Plans_for_Cont_Env}

In this section we consider collections 
of planning problems in which a point robot explores a compact, connected, planar state space $X \subset \mathbb{R}^2$ with non-empty interior $\textrm{int}X$. The aim is to reach an open $r$-neighborhood $X_G := B_r(x_G) \cap \textrm{int}X$ around some goal state $x_G \in \textrm{int}X$, from the initial state $x_I \in \textrm{int}X$.
We refer to $r$ here as the \emph{goal radius}.
We show that there exists a universal plan that allows the robot to solve this task essentially without sensing, in any environment $X$ with suitably regular boundary, for any $x_I, x_G \in \textrm{int}X$ and any $r>0$.

 Assume the robot is able to take arbitrarily small steps in all the coordinate directions (\verb|left|, \verb|right|, \verb|up|, and \verb|down|). The set of actions $U$ is thus given by
\begin{equation*} 
U := \big\{ sv \mid s \in \mathbb{R}_+, \,\, v \in \{v_\textrm{l}, v_\textrm{r}, v_\textrm{u}, v_\textrm{d}\} \big\},
\end{equation*}
where $\{v_\textrm{l}, v_\textrm{r}, v_\textrm{u}, v_\textrm{d}\} := \{(-1,0), (1,0), (0,1), (0,-1)\}$.
The transition function $f : X \times U \to X$ is 
\begin{equation} \label{Eq_robot_movement_model_again}
f(x, u) := \begin{cases}
    x + u, & \textrm{if} \,\,  x + u \in X, \\
    x, & \textrm{if} \,\,  x + u \notin X.
\end{cases}
\end{equation}
Thus, the robot does not move if it is about to hit the environment boundary.

Assume for the moment that the step size of the robot is fixed to some $s > 0$. If the step size is too large, the robot may be unable to pass through narrow corridors, or it might miss the target area $B_r(x_G)$ around the goal state. We next analyze the existence of sufficiently fine step sizes for a given planning problem.

\subsection{Finite Discretizations for Planar Planning Problems} \label{Sec_Finite_Discretizations_for_Planar}
Denote by $\mathcal{C}^\circ(\R^2)$ the collection of connected, compact subsets of $\R^{2}$ with non-empty interior. To analyze the sufficiency of the movement resolution, we associate with every triple $(X, x_I, m) \in \mathcal{C}^\circ(\R^2) \times \textrm{int}X \times \N$ the corresponding grid
\begin{equation} \label{Eq_Def_Env_Discretization_Grid}
X_\textrm{grid}(x_I, m) := \mathrm{int}X \cap \big(x_I + 2^{-m} (\mathbb{Z} \times \mathbb{Z})\big).
\end{equation}
Intuitively, the grid $X_\textrm{grid}(x_I, m)$ is a discretization of $X$ at resolution $2^{-m}$. Ideally, such a grid is connected in the sense of Definition~\ref{Def_Connected_Grid}, allowing the robot to reach the goal from any initial state. However, this is not generally the case.

\begin{definition}[\textbf{Globally connected grid representation (GCGR); suf\-fi\-cient scaling factor}] \label{Def_Globally_connected_representation_and_suff_scaling_factor}
An environment $X \in \mathcal{C}^\circ(\R^2)$ \emph{admits a globally connected grid representation (GCGR)}, if there exists some $\eta(X) \in \N$ such that for all $m \geq \eta(X)$ and all initial states $x_I \in X$, the grid discretization $X_\textrm{grid}(x_I, m)$ defined in~\eqref{Eq_Def_Env_Discretization_Grid} is connected in the sense of Definition~\ref{Def_Connected_Grid}. We then call $\eta(X)$ a  \emph{sufficient scaling factor} for $X$.
\end{definition}

A natural step would be to characterize all the environments that admit a GCGR. Proposition~\ref{Prop_sufficient_independent_scaling_exists} in Section~\ref{sec:scale} states that all environments whose complement is a set with positive reach admit a GCGR. For example, environments with a differentiable boundary with bounded curvature satisfy this condition. Unfortunately, not all polygonal environments admit such a representation, since they may exhibit sharp wedges where initial states near the tip of the wedge require increasingly fine grid scaling to obtain connectedness.

\begin{figure}[tb]
\begin{tabular}{ccc}
\includegraphics[width=4cm]{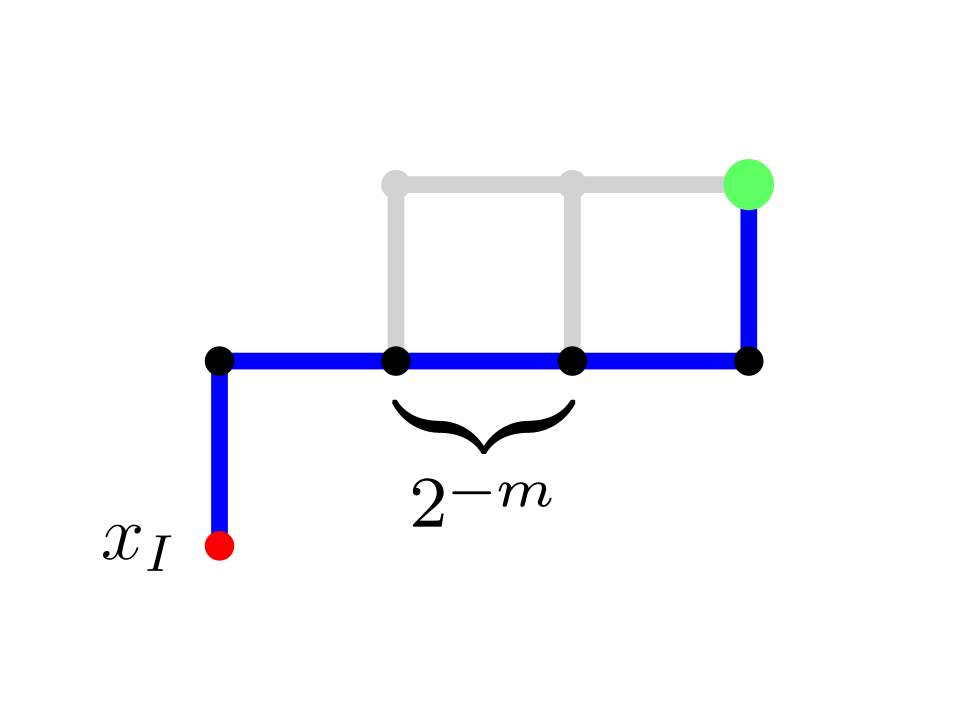} &
\includegraphics[width=4cm]{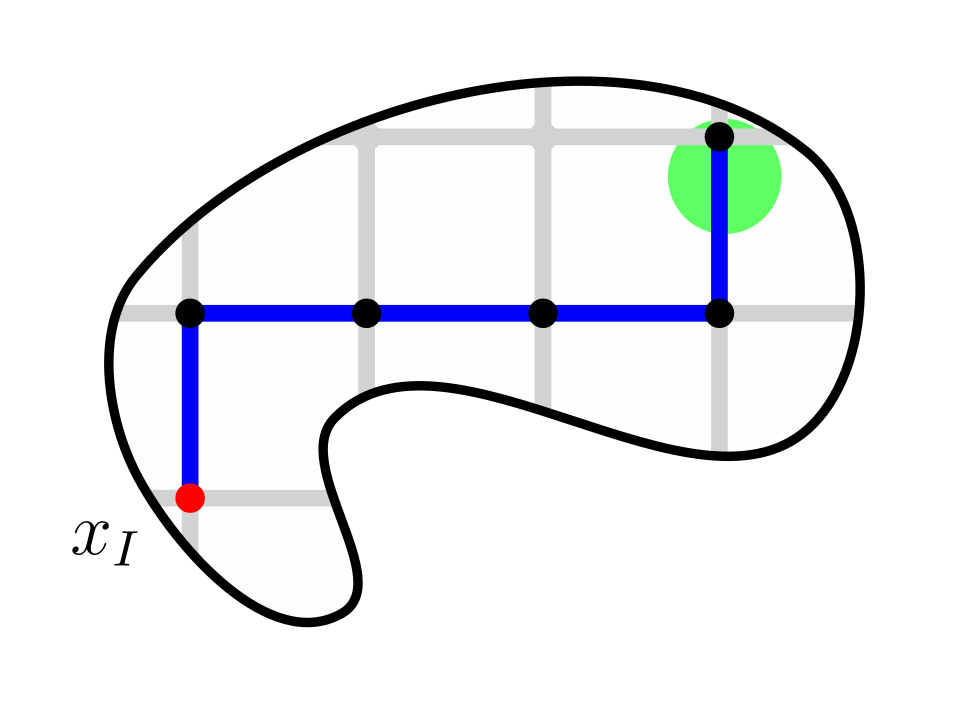} &
\includegraphics[width=4cm]{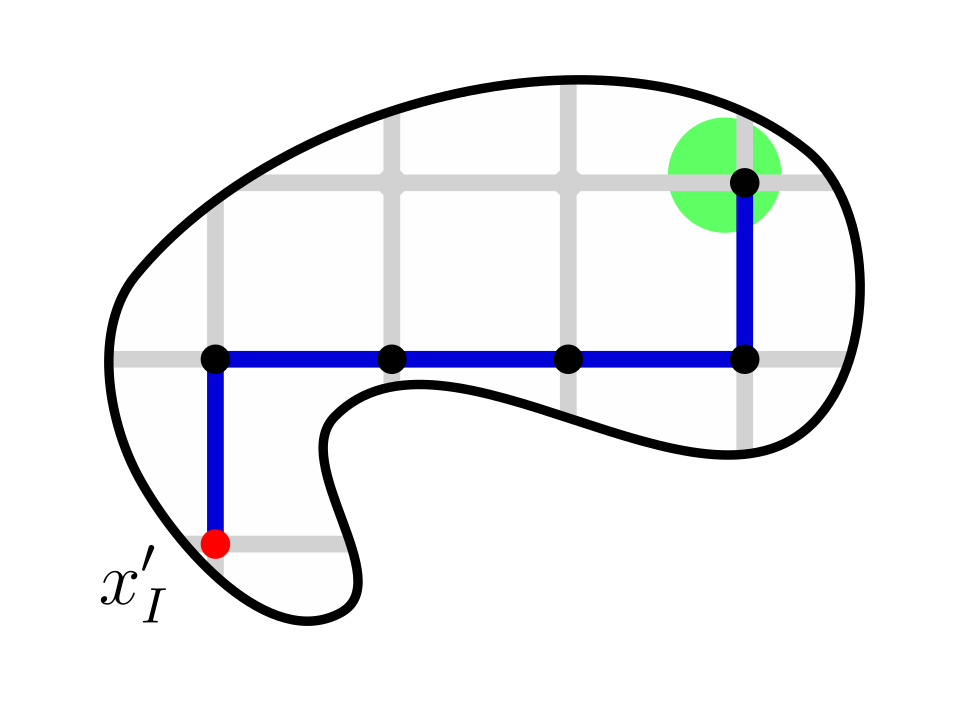} \\
a. & b. & c.
\end{tabular}
\caption{\label{fig:grid_plans} Discretization of the planar planning problem. (a) The discretization $X_\textrm{grid}(x_I, m) \subset \big(x_I + 2^{-m} (\mathbb{Z} \times \mathbb{Z})\big)$ at scaling resolution $2^{-m}$. The horizontal and vertical lines connecting the vertices are drawn here to indicate possible transitions, but are not included in $X_\textrm{grid}(x_I, m)$. The dark blue line indicates a possible plan taking the initial state (red dot) to the goal state (green dot). (b) \& (c) Two possible ways the same grid can emerge from two different initial states $x_I, x_I' \in \textrm{int}X$.
} 
\end{figure}

Let $\Pglobal^m$ be the set of planning problems $P_X = (X, U, f, x_I, X_G)$ in which the environment $X$ admits a GCGR with a sufficient scaling factor $m$, $X_G = B_r(x_G) \cap \textrm{int}X$ for some $x_G \in \textrm{int}X$, $r > 0$, and let $\Pglobal := \bigcup_{m=1}^\infty \Pglobal^m$. In Section~\ref{sec:scale} we show that there exist plans $\histu$ that are universal with respect to $\Pglobal$. To this end, we next analyze,
for a fixed $X \in \CCC$ and $m \in \N$, the set
\begin{equation} \label{Eq_Def_Set_of_grids_at_a_scale}
\mathcal{G}(X, m) := \big\{X_\textrm{grid}(x_I, m) \, \mid \, x_I \in \textrm{int}X \big\}
\end{equation}
of grid representations~\eqref{Eq_Def_Env_Discretization_Grid} associated with all possible initial states $x_I \in X$. The key observation is that although there are uncountably many possible initial states $x_I \in \textrm{int}X$, once the resolution $m$ is fixed, each of them corresponds to a position on a finite grid belonging to a finite family of grids, each one approximating $X$ at this resolution, see Figure~\ref{fig:grid_plans}.

\begin{lemma}[\textbf{Finitely many grids}] \label{Lemma_finitely_many_grids}
Let $X \in \CCC$. Then for each $m \in \N$, the set $\mathcal{G}(X, m)$ in~\eqref{Eq_Def_Set_of_grids_at_a_scale} contains up to translation only finitely many unique grids, which all have finitely many states. If $X$ admits a GCGR and $m \geq \eta(X)$, then all the grids in $\mathcal{G}(X, m)$ are connected in the sense of Definition~\ref{Def_Connected_Grid}.
\end{lemma}

\begin{proof}
Fix $m \in \N$. For each initial state $x_I \in X$ consider the translation 
defined by $T_{x_I}(x) := x - x_I$. Note that $X$ has finite width and height $W, H \in \R_+$ since it is compact. It follows that $T_{x_I}(\textrm{int}X) \subset [-W, W]\times [-H, H]$ for all $x_I \in X$. The first claim then follows from the fact that the power set of the intersection $2^{-m}(\Z \times \Z) \cap [-W, W]\times [-H, H]$ is finite, and the observation that the sets $X_\textrm{grid}(x_I, m)$ and $\big(T_{x_I}(X)\big)_\textrm{grid}\big((0,0), m\big)$ are in one-to-one correspondence via $T_{x_I}$. The last claim regarding connectedness follows directly from the definition of the sufficient scaling factor $\eta(X)$. $\square$
\end{proof}

Two grids $X_\textrm{grid}(x_I, m)$, $X_\textrm{grid}(x_I', m)$ are \emph{translation equivalent} if there exists some translation $T_z$ for which $X_\textrm{grid}(x_I, m) = T_z\big(X_\textrm{grid}(x_I', m)\big)$.
We then write
\begin{equation} \label{Eq_Def_Grid_equivalence}
X_\textrm{grid}(x_I, m) \sim_T X_\textrm{grid}(x_I', m).
\end{equation}
Note that we do not require in the above definition that $x_I = T_z(x_I')$.
Having this in mind, we say that two states $x, x' \in X$ are \emph{grid-equivalent at resolution $m \in \N$}, $x \sim_E x'$, if the translation implied by~\eqref{Eq_Def_Grid_equivalence} can be chosen to be $x' - x$:
\begin{equation}
x \sim_E x' \,\,\Longleftrightarrow \,\, X_\textrm{grid}(x, m) = T_{x' - x}\big(X_\textrm{grid}(x', m)\big).
\end{equation}
The individual grids $X_\textrm{grid}(x_I, m) \in \mathcal{G}(X,m)$ can thus be divided into finitely many equivalence classes $\big[X_\textrm{grid}(x_I, m)\big]_{\sim_T} \in \mathcal{G}(X,m) /\sim_T$, and to each initial state $x_I \in X$ we may associate the equivalence class $[x_I]_{\sim_E} \in \big[X_\textrm{grid}(x_I, m)\big]_{\sim_T}$. As the robot moves using a particular step size, it effectively explores the equivalence class $\big[X_\textrm{grid}(x_I, m)\big]_{\sim_T}$ from the initial state $[x_I]_{\sim_E}$. We now clarify when such explorations can be regarded equivalent as planning problems.

\begin{definition}[\textbf{Isomorphic planning problems}] \label{Def_Isomorphic_Planning_Problems}
Planning problems $P := (X, U, f, x_I, X_G)$ and $P' := (X', U', f', {x_I}', X_G')$ are \emph{isomorphic} if there exist bijective mappings $\varphi: X \to X'$ and $\xi: U \to U'$ for which (i) $\varphi(x_I) = {x_I}'$, (ii) $\varphi(X_G) = X_G'$, and (iii) $f'(\varphi(x), \xi(u)) = \varphi(f(x, u))$ for all $x \in X$ and $u \in U$.
\end{definition}

An isomorphism preserves the structure that is relevant for goal attainment. For instance, it can be readily checked that some action sequence $\histu$ satisfies $x_I\cat \histu \in X_G$ if and only if ${x_I}' \cat \xi(\histu) \in X_G'$.

Consider a planning problem $(X, U, f, x_I, X_G) \in \Pglobal$, so that $X$ admits a GCGR, and let $r > 0$ be the goal radius. For any $X_\textrm{grid}(x_I, m) \in \mathcal{G}(X,m)$ with $m \geq \eta(X)$, define the corresponding discretized goal set by 
\begin{equation} \label{Eq_Def_Discretized_goal_set}
x_G(x_I, m) :=
B_r(x_G) \cap \big(x_I + 2^{-m} (\mathbb{Z} \times \mathbb{Z})\big).
\end{equation}
It follows from the connectedness of $X_\textrm{grid}(x_I,m)$ that for sufficiently large $m$ the set $x_G(x_I, m)$ is non-empty for any initial state $x_I$. Consider then the family $\mathcal{P}_\textrm{grid}\big(X, x_G, m\big) := \big\{P_\textrm{grid}(X, x_I, x_G, m) \mid x_I \in X\big\}$ of discretized planning problems, where, for each $x_I \in X$,
\begin{equation}
P_\textrm{grid}(X, x_I, x_G, m) := \big(X_\textrm{grid}(x_I, m), U, f, x_I, x_G(x_I, m)\big).
\end{equation}

\begin{definition}[\textbf{Goal and grid search equivalence}] \label{Def_Goal_and_Grid_search_equivalence}
Let $X \in \CCC$, assume $X$ admits a GCGR and let $r > 0$. For any initial conditions $x_I, x_I' \in X$, the discretized planning problems $P_\textrm{grid}(X, x_I, x_G, m), P_\textrm{grid}(X, x_I', x_G, m) \in \mathcal{P}_\textrm{grid}(X, x_G, m)$
are said to be \emph{goal-equivalent}, denoted $P_\textrm{grid}(X, x_I, x_G, m) \sim_G P_\textrm{grid}(X, x_I', x_G, m)$, if there exists some translation $z \in \R^2$ for which
\begin{equation} \label{Eq_Def_Goal_equivalence}
X_\textrm{grid}(x_I', m) = T_z\big(X_\textrm{grid}(x_I, m)\big), \,\,\, \textrm{and}
\,\,\,\,  x_G(x_I', m) = T_z\big(x_G(x_I, m)\big).
\end{equation}
In other words, the two grids are translation equivalent and the goal set $x_G(x_I, m)$ maps to $x_G(x_I', m)$ under the same translation.
If~\eqref{Eq_Def_Goal_equivalence} holds with $z = x_I' - x_I$ so that the initial states are grid-equivalent, we say that the planning problems are \emph{grid search-equivalent}, and denote $P_\textrm{grid}(X, x_I, x_G, m) \sim_S P_\textrm{grid}(X, x_I', x_G, m)$.
\end{definition}
Note that goal-equivalent planning problems are generally not isomorphic in the sense of Definition~\ref{Def_Isomorphic_Planning_Problems} because their initial states may differ, and thus affect the set of successful plans. However, it can be readily checked that the maps $\varphi = T_{x_I' - x_I}$ and $\xi = \textrm{Id}$ satisfy the properties stated in Definition~\ref{Def_Isomorphic_Planning_Problems}, and thus define an isomorphism between grid search-equivalent planning problems.

\begin{proposition}[\textbf{Finitely many non-isomorphic planning problems}] \label{Prop_Finitely_Many_Nonisomorphic_Planning_Probs}
Assume $X \in \CCC$ admits a GCGR, and let $r > 0$. Then $\mathcal{P}_\textrm{grid}(X, x_G, m)$ contains only finitely many non-isomorphic planning problems for each $m \geq \eta(X)$.
\end{proposition}

\begin{proof}
According to Lemma~\ref{Lemma_finitely_many_grids}, the set of equivalence classes $[X_\textrm{grid}(x_I, m)]_{\sim_T}$ is finite. Moreover, each representative of such a class has finitely many states. Then clearly there are, for a fixed grid equivalence class $[X_\textrm{grid}(x_I, m)]_{\sim_T}$ only finitely many possible combinations between equivalence classes $[x_I']_E \neq [x_I]_E$ of alternative initial states $x_I' \in X_\textrm{grid}(x_I',m)$
with $X_\textrm{grid}(x_I',m) \sim_T X_\textrm{grid}(x_I,m)$, and the corresponding sets $\big\{ [x]_E \mid x \in x_G(x_I', m)\big\}$ of equivalence classes corresponding to the elements of the associated discretized goal sets. $\square$
\end{proof}


\subsection{A Scale-free Action Sequence}\label{sec:scale}
In this section we construct a universal plan for the family of planning problems $\Pglobal$, defined in Section~\ref{Sec_Finite_Discretizations_for_Planar}.
The problem seems challenging at first, given that there are uncountably many environments $X \subset \R^2$ for which the plan must apply. Our approach is to reduce each planning problem in~$\Pglobal$ to a corresponding robot grid search problem $\Pgrid (\Z \times \Z)$, and show how these can be sampled in an effective way to cover the whole family $\Pglobal$.

Assume the robot has access to a normal sequence $\alpha_1\alpha_2\ldots$ in base $4$, which can be used to decide the direction and step size of the next action. We restrict the step size to the countable set $A := \{2^{-m} \mid m = 0,1,2, \ldots\}$. The step sizes could be chosen either (i) by specifying some rule $\gamma: \N \to A$ for choosing the sizes
independently of the $\alpha_k$, or (ii) by specifying some rule $\xi: \{0,1,2,3\} \to A$ that determines for each $k\in\N$ the next step size $\xi(\alpha_k)$ from the current decimal $\alpha_k$. We present a provably universal plan using approach (i). In Section~\ref{sec:sim} we also experiment with an alternative plan that uses approach (ii), and empirically outperforms the first approach, but we have not proved its universality.

For $w > 0$, define the mapping $L_w: \N \to \N$ recursively by $L_w(1) := 1$, and
\begin{equation} \label{Eq_Def_length_of_step_size_segment}
L_w(n) := \ceil[\big]{w \big(L_w(1) + \ldots + L_w(n-1)\big)}
\end{equation}
for all $n \geq 2$. Each $L_w(n)$ represents the number of steps the robot takes before the next change in step size. The weight $w$ indicates the ratio of how many steps the robot takes using a particular step size relative to how many steps it has taken in total up to that point. We next define the order in which different step sizes will be chosen during the exploration. Define $\beta_n := \max \big\{ k \mid  k(k+1)/2 \leq n \big\}$ and $\varphi(n) := n - \beta_n(\beta_n+1)/2$. For instance, $(\beta_n)_{n=1}^{9} = (1, 1, 2, 2, 2, 3, 3, 3, 3)$ and $\varphi(\beta_1, \ldots, \beta_{9}) = (0, 1, 0, 1, 2, 0, 1, 2, 3)$. The algorithm thus samples different step sizes by starting repeatedly from $1 = 2^0$ and proceeding to smaller and smaller step sizes $2^{-m}$.
Then define $\eta_w(n) := \max\big\{k \mid \sum_{j=1}^k L_w(k) \leq n\big\}$, and let $c: \{1,2,3,4\} \to \{v_\textrm{l}, v_\textrm{r}, v_\textrm{u}, v_\textrm{d}\}$ map the normal number digits to the corresponding unit-length steps. The function $\gamma_{\alpha, w}: \N \to \N$ that gives the complete action sequence, sampling from the normal sequence $\alpha_1 \alpha_2 \alpha_3\ldots$, is then given by
\begin{equation} \label{Eq_Def_Universal_Sequence}
\gamma_{\alpha, w}(n) := 2^{-(\varphi \circ \eta_w)(n)} c(\alpha_n).
\end{equation}
Intuitively, $\eta_w(n)$ gives the $n$:th index in the step size alteration scheme and is converted by $\varphi$ to the appropriate power $m \in \N$.

\subsection{Proof of Universality of the Scale-free Plan} \label{sec:proof_of_universal_scale_free}
Consider some planning problem $P = (X,U,f,x_I,X_G) \in \Pglobal$, so that the environment $X \in \mathcal{C}^\circ(\R^2)$ admits a GCGR, see Definition~\ref{Def_Globally_connected_representation_and_suff_scaling_factor}. Using the algorithm described above, the robot starts to explore the environment along the grid $X_\textrm{grid}(x_I, 1)$ which corresponds to step size $1/2$. Unless it immediately moves to within the required goal radius from the goal state $x_G$, it will switch to the next finer step size $2^{-2}$ and continue. As the robot traverses using finer step sizes and again reverts back to coarser ones, it effectively changes the planning problem from one grid $X_\textrm{grid}(x_I, m)$ to another $X_\textrm{grid}(x_I', m')$, where in general both the grid resolution $m'$ and the initial state $x_I'$ have changed.

However, according to Lemma~\ref{Lemma_finitely_many_grids} there exist for each step size $2^{-m}$ only finitely many different grids $\mathcal{G}(X, x_I, m)$. In addition, since $X$ admits a GCGR, there exists some $\eta(X)$ for which each grid $X_\textrm{grid}(x_I, m)$ with $m \geq \eta(X)$ is connected. From this it follows that for all large enough $m > \eta(X)$, the grid $X_\textrm{grid}(x_I, m)$ intersects the open set $B_r(x_G) \cap \textrm{int}X$. If the robot knew to stop switching the step size once it reached one of these grids, it would eventually reach some $x \in B_r(x_G) \cap X_\textrm{grid}(x_I, m)$ in finite time according to Corollary~\ref{Cor_rich_tries_everything}. Fortunately, it turns out that it is sufficient for the robot to keep indefinitely revisiting some step size $2^{-m}$ with $m \geq \eta(X)$.

\begin{theorem}[\textbf{Universal Planar Search Plan}] \label{Thm_Univ_Planar_Search_Algo}
For every normal sequence $\alpha_1\alpha_2\alpha_3\ldots$ and $w > 0$, the scale-free search plan $\gamma_{\alpha,w}$ in~\eqref{Eq_Def_Universal_Sequence} is universal with respect to the family of planning problems $\Pglobal$.
\end{theorem}
Before proving the theorem itself, we establish a final auxiliary result. It states that subsequences $(\alpha_{n_k})_{k=1}^\infty$, associated with robot actions $\gamma_{\alpha,w}(n_k)$, 
which take place within some set $[P]$ of goal-equivalent discretized planning problems $P \in \mathcal{P}_\textrm{grid}(X, x_G, m)$ are themselves normal sequences. For this, we need to introduce the concept of a \emph{completely deterministic sequence}~\cite{BerDowVan22,BerVan19,Wei99}. The combinatorial definition below has been shown to be equivalent to the measure theoretic definition~\cite[Definition 8.8]{Wei99}. The latter was used in~\cite{Kam73} to establish the full characterization of selection rules that preserve normality. Below, the \emph{indicator sequence} $(\omega_n)_{n=1}^\infty \in \{0,1\}^\N$ of a sequence $(n_k)_{k=1}^\infty \subset \N$ is defined by setting $\omega_n = 1$ if and only if $n = n_k$ for some $k \in \N$.

\begin{definition}[\textbf{Completely deterministic sequence (Weiss)}] \label{Def_Completely_Deterministic}
A binary sequence (selection rule) $(\omega_n)_{n=1}^\infty$
is \emph{completely deterministic} if and only if for any $\varepsilon > 0$ there is some $K$ such that after removing from $(\omega_n)_{n=1}^\infty$ a subset of density less than $\varepsilon$, what is left can be covered by a collection $\mathcal{C}$ of $K$-blocks such that $|\mathcal{C}| < 2^{\varepsilon K}$. An increasing sequence $(n_k)_{k=1}^\infty$ of positive integers is completely deterministic if its indicator sequence $(\omega_n)_{n=1}^\infty$ is completely deterministic.
\end{definition}
Intuitively, in a completely deterministic sequence, only a small collection $\mathcal{C}$ of strings (K-blocks), all with a fixed length $K$, appear very frequently, and all other strings extremely rarely, as calibrated freely with the parameter $\varepsilon > 0$.

\begin{lemma}[\textbf{Step size subsequences are completely deterministic}] \label{Lemma_Step_Sizes_Comp_Deterministc}
For each $m \in \N$ and $w > 0$, let $(n_k)_{k=1}^\infty$ be the increasing sequence of stages for which $(\varphi \circ \eta_w)(n_k) = m$, i.e.~the robot uses the step size $2^{-m}$ precisely at stages $n_k$. Then
$(n_k)_{k=1}^\infty$ is a completely deterministic sequence.
\end{lemma}

\begin{proof}
Let $(\omega_n)_{n=1}^\infty$ be the indicator sequence of $(n_k)_{k=1}^\infty$.
It follows from the definition of the sets $L_w(n)$ in~\eqref{Eq_Def_length_of_step_size_segment} that asymptotically the sequence $(\omega_n)_{n=1}^\infty$ contains longer and longer streaks of $0$s, separated by longer and longer streaks of $1$s. Thus, for a fixed $K \in \N$, the only $K$-blocks that appear infinitely often in $(\omega_n)_{n=1}^\infty$ consist either of a streak of $0s$ of length $K-j$ followed by a streak of $1$s of length $j$, or vice versa, where $j$ runs through $0, \ldots, K$.
For $K = 3$ the possible sequences are $\{000, 001, 011, 111, 110, 100\}$. The number of $K$-blocks that appear infinitely many times in $(\omega_n)_{n=1}^\infty$ is thus $2K$ for any $K$.

Let then $\varepsilon > 0$ and choose $K$ to be such that $(1 + \log_2 K)/K < \varepsilon$, from which it follows that $2K < 2^{\varepsilon K}$. Now, let $\mathcal{C}$ be the set of all length-$K$ sequences that appear in $(\omega_n)_{n=1}^\infty$ infinitely many times. From the above reasoning, $|\mathcal{C}| = 2K$. According to~\eqref{Eq_Def_length_of_step_size_segment}, the length of the alternating blocks of $1$s and $0$s in $(\omega_n)_{n=1}^\infty$ will exceed $K$ from index $n = K/w$ onwards. After this, the entire sequence can be covered using only sequences in $\mathcal{C}$. Since the finite initial segment before $n = K/w$ has asymptotic density zero, the result follows. $\square$
\end{proof}

A sequence $(n_k)_{k=1}^\infty \subset \N$ \emph{preserves normality}, if for every normal sequence $(\alpha_n)_{n=1}^\infty$, the subsequence $(\alpha_{n_k})_{k=1}^\infty$ is also normal. Before presenting the proof of Theorem~\ref{Thm_Univ_Planar_Search_Algo}, we restate the central characterization result due to Kamae~\cite{Kam73} and Weiss~\cite{Wei71} as given in~\cite{BerVan19}.

\begin{theorem}[\textbf{Admissible selection rules (Kamae, Weiss)}] \label{Thm_Kamae_Weiss}
An increasing sequence of positive integers $(n_k)_{k=1}^\infty$ preserves normality if and only if it is a completely deterministic sequence of positive lower asymptotic density.
\end{theorem}

Equipped with Theorem~\ref{Thm_Kamae_Weiss}, we are set to prove our main result, Theorem~\ref{Thm_Univ_Planar_Search_Algo}.

\begin{proof}[Theorem~\ref{Thm_Univ_Planar_Search_Algo}]
We first show that given a particular planning problem $P(X, U, f, x_I, x_G) \in \Pglobal$, we can associate with every $m \geq \eta(X)$ a finite sequence $\histu(m)$ that is universal with respect to the family $\mathcal{P}_\textrm{grid}(X, x_G, m)$ of discretized planning problems. Let $m \geq \eta(X)$. Proposition~\ref{Prop_Finitely_Many_Nonisomorphic_Planning_Probs} states that there are only finitely many, say $M$, equivalence classes $[P_k]_{\sim_S} \in \mathcal{P}_\textrm{grid}(X, x_G, m) / \sim_S$. Within a class, a plan successful for $P_k$ is successful for all $P \sim_S P_k$. According to Corollary~\ref{Cor_rich_tries_everything}, we may choose for each $[P_k]_{\sim_S}$ some plan $\tilde{u}^{(k)} = \big(u_1^{(k)}, \ldots, u_{n_k}^{(k)}\big)$ which succeeds from any initial state $x_I' \in \textrm{int}X$ satisfying $P_\textrm{grid}(X, x_I', x_G, m) \sim_S P_k$. Hence, $\tilde{u}(m) := \tilde{u}^{(1)}\cat \cdots \cat \tilde{u}^{(M)}$ is successful for any $P \in \mathcal{P}_\textrm{grid}(X, x_G, m)$.

We now show that for all $m \geq \eta(X)$, 
the robot will eventually execute the plan $\histu(m)$ using only steps of size $2^{-m}$. According to Lemma~\ref{Lemma_Step_Sizes_Comp_Deterministc}, the subsequence $(n_k)_{n=1}^\infty$ defined by $(\varphi \circ \eta_w)(n_k) = m$ is completely deterministic, and it has lower asymptotic density $w$ by definition. This, together with Theorem~\ref{Thm_Kamae_Weiss} implies that the corresponding subsequence $(\alpha_{n_k})_{k=1}^\infty$ is normal. This in turn implies that $\histu(m)$ appears in the sequence $(c(\alpha_{n_k}))_{k=1}^\infty$ with relative frequency equal to all other blocks of the same length. Assume then, contrary to the claim, that the robot never executes the whole sequence $\histu(m) = \big(u_1, \ldots, u_{l(m)}\big)$ while taking steps of size $2^{-m}$. From~\eqref{Eq_Def_length_of_step_size_segment} it follows that the lengths of continuous blocks in the sequence $(n_k)_{k=1}^\infty$ monotonically increase without bounds, implying that the indices $k$ for which $\histu(m) = \big(c(\alpha_k), \ldots, c(\alpha_{k + l(m)})\big)$ must always lie within $l(m)$ from the end of such blocks. But this implies that the lower asymptotic density of the sequence $\histu(m)$ is zero, which is a contradiction. $\square$
\end{proof}

\begin{figure}[tb]
\begin{tabular}{cc}
\includegraphics[width=6cm]{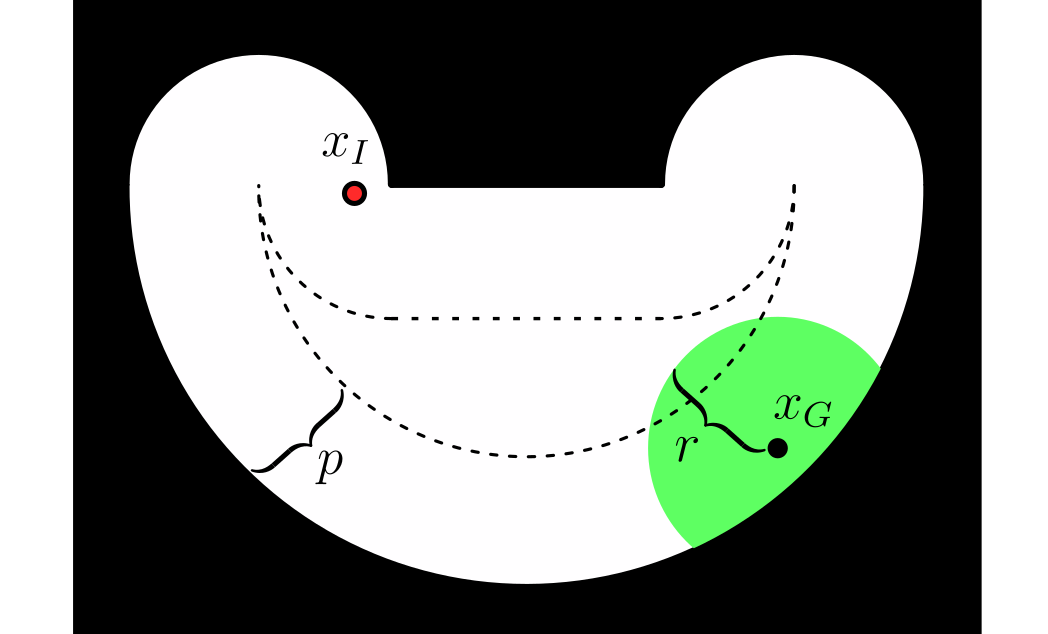} &
\includegraphics[width=6cm]{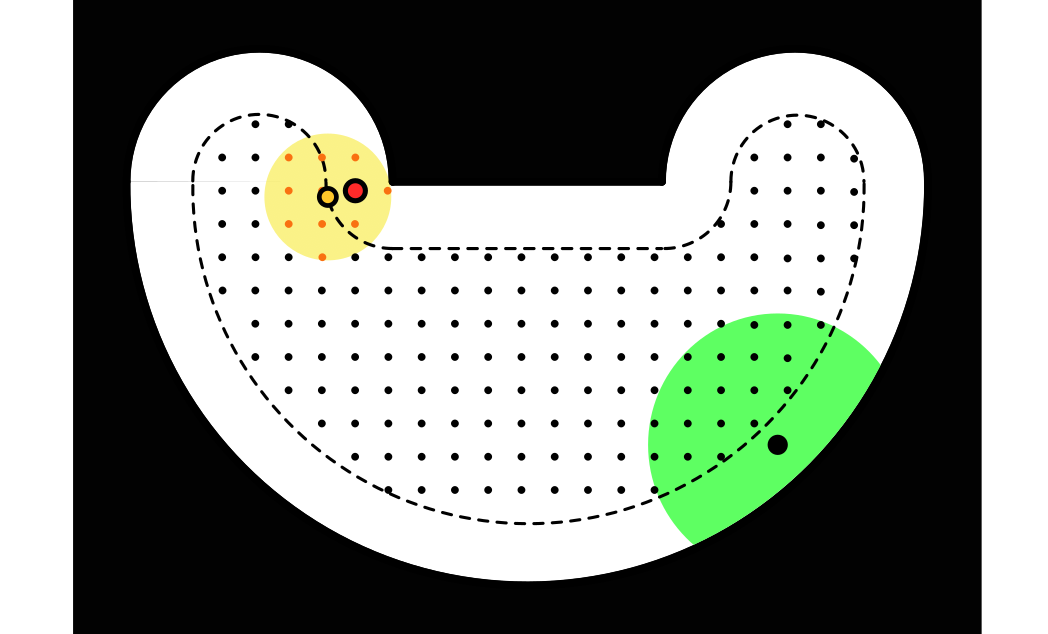} \\
a. & b.
\end{tabular}
\caption{\label{fig:positive_reach} Determining the sufficient scaling factor $\eta(X,r)$ in an environment whose complement is a set with positive reach. (a) Radius $r$ of the goal area $B_r(x_G)$, and $p = \textrm{reach}(\R^2 \setminus X)$. (b)
The grid $V(X, x_I, m) := X_\textrm{grid}(x_I, m) \cap X_{p/2}$ in which $m = p/4$. Note that $x_I \notin V(X, x_I, m)$, but $x_I \in B_{p/2}(z_p)$ in which $z_p$ (\includegraphics[width=0.22cm]{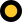}) minimizes the distance from  $x_I$ to the set $X_{p/2}$.
} 
\end{figure}

\subsection{Characterizing Environments that Admit GCGR}
Here we provide a simple sufficient property which guarantees that a given environment $X$ admits a GCGR, see Definition~\ref{Def_Globally_connected_representation_and_suff_scaling_factor}.
These are sets $X \in \R^2$ whose complement $\mathbb{R}^2 \setminus X$ is a set with \emph{positive reach}~\cite{RatZah2019}. Essentially, a set has positive reach if
a ball with some fixed radius $r$ can be rolled on the boundary $\partial E$ so that every point $x \in \partial E$ is eventually touched by the ball. The boundary $\partial X$ of such an environment could be a smooth curve with bounded curvature, or a union of polygons with no wedges. See Figure 3 in supplementary material.

\begin{proposition}[\textbf{Sufficient scaling factor independent of start state}] \label{Prop_sufficient_independent_scaling_exists}
Let $X \subset \mathbb{R}^2$ be a compact, connected environment whose complement is a set with positive reach, and let $x_G \in X$ be a goal state. Then, for every goal radius $r>0$ there exists a sufficient scaling factor $\eta(X, r) > 0$ such that for all $m > \eta(X, r)$ and all initial states $x_I \in \mathrm{int} X$, the grid $X_\textrm{grid}(x_I, m) := \mathrm{int}X \cap \big(x_I + 2^{-m} (\mathbb{Z} \times \mathbb{Z})\big)$ is connected and intersects the ball $B_r(x_G)$.
\end{proposition}

\begin{proof}
Let $p := \textrm{reach}(\mathbb{R}^2 \setminus X)$, let $r > 0$ and let $s > 0$ be such that $2^{-s} = \min\{p, r\} / 4$. Define $\eta(X, r) = s$, choose any $s < m \in \N$ and consider the subgrid
\[
V(X, x_I, m) := X_\textrm{grid}(x_I, m) \cap X_{p/2},
\]
in which $X_{p/2} := X \setminus \overline{B_{p/2}(\mathbb{R}^2 \setminus X)}$. This set consists of the grid points that are further than $p/2$ from the boundary $\partial X$. Since $X$ is connected, it follows that $V(X, x_I, m)$ is a connected grid, but it might be the case that $x_I \notin V(X, x_I, m)$.

We now show that every $z \in X_\textrm{grid}(x_I, m)$ can be connected to $V(X_, x_I, m)$ along a chain of adjacent points in $X_\textrm{grid}(x_I, m)$.
Let $z_p := \textrm{argmin}_{q \in \overline{X_{p/2}}} \textrm{dist}(q, z)$. Since $z_p \in \overline{X_{p/2}}$, it follows that $B_{p/2}(z_p) \subset \textrm{int}X$. Furthermore, $z \in \textrm{int}X$ implies that $z \in B_{p/2}(z_p)$, see Figure~\ref{fig:positive_reach}. Since $2^{-m} < 2^{-s} \leq p/4$, it follows that the subgrid
\[
W_z := B_{p/2}(z_p) \cap \big(x_I + 2^{-m} (\mathbb{Z} \times \mathbb{Z})\big)
\]
is connected and satisfies $z \in W_z$ and $W_z \cap V(X, x_I, m) \neq \varnothing$, as required. $\square$
\end{proof}

\section{Learning Optimal Plans}

Clearly, action sequences generated from universal plans are far from optimal.  In this section, we briefly sketch how executing a universal plan can nevertheless produce optimal plans.  We would like the robot to record the shortest possible sequence of actions that brings it from the initial state $x_I$ to the goal set $X_G$. We expect the following to generalize to optimality for any stage-additive cost model with minor modifications.

Suppose the robot can detect initial and goal states.  Let $h_\textrm{init} : X \rightarrow \{0,1\}$ be an {\em initial-state detector}, which yields $h(x) = 1$ if and only if 
$x = x_I$.
Similarly, let $h_\textrm{goal} : X \rightarrow \{0,1\}$ be a {\em goal detector} with $h(x) = 1$ if and only if $x \in X_G$.

\begin{proposition}[\textbf{Learning optimal plans}] \label{Prop_optimal_plan} Let $P_X = (X, U, f, x_I, X_G)$ be a robot grid search problem. For a robot with initial and goal state detectors and unlimited memory, an algorithm exists for which any universal plan causes the robot to learn the optimal plan after a finite number of steps.
\end{proposition}

\begin{proof}
An anytime algorithm that returns the optimal plan is produced as follows.  It maintains the shortest action sequence that starts and ends at $x_I$ and $X_G$, respectively, that has been encountered so far.  Initially, it is empty.  In another buffer, if $h_\textrm{init} = 1$, then it starts recording the sequence of actions taken until either $h_\textrm{init} = 1$ again, in which case it resets the recording, or $h_\textrm{goal} = 1$.  If the sequence in the buffer is shorter than the shortest one obtained so far, then it is kept as the shortest. Corollary~\ref{Cor_rich_tries_everything} guarantees that every finite-length sequence is visited in the action sequence, from every reachable state. Hence, the optimal plan (finite action subsequence) must be visited at some step. At this point, the robot has recorded the optimal plan. $\square$
\end{proof}

Note that the robot cannot ever report that the optimal plan has been learned.  Only an external observer with complete information would know if the stored best plan so far happens to be optimal.  Furthermore, it still cannot decide whether a problem is solvable.  If not, the robot explores forever in vain. Note also that when an optimal plan exists, the robot does not learn it as a limit of converging approximations, but simply stumbles upon it after a finite time.

\section{Simulations}\label{sec:sim}

\begin{figure}[tb]
\begin{tabular}{ccc}
\includegraphics[width=4cm]{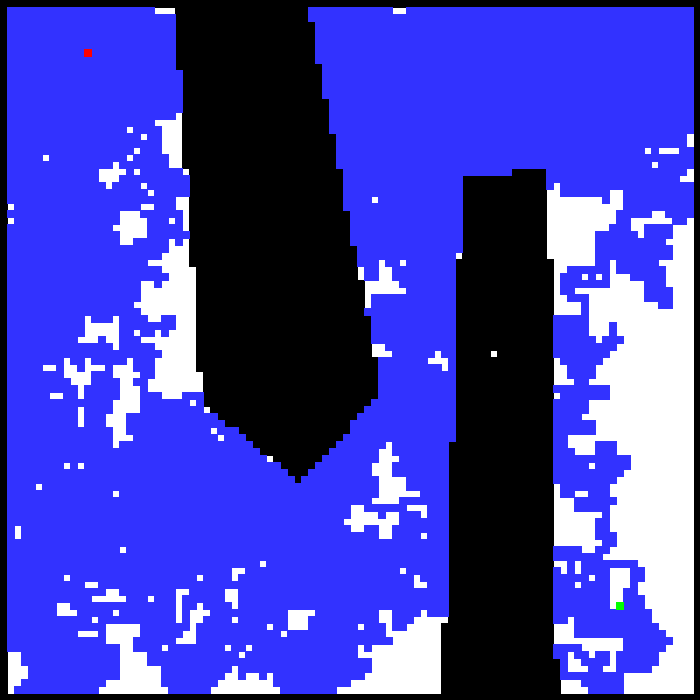} &
\includegraphics[width=4cm]{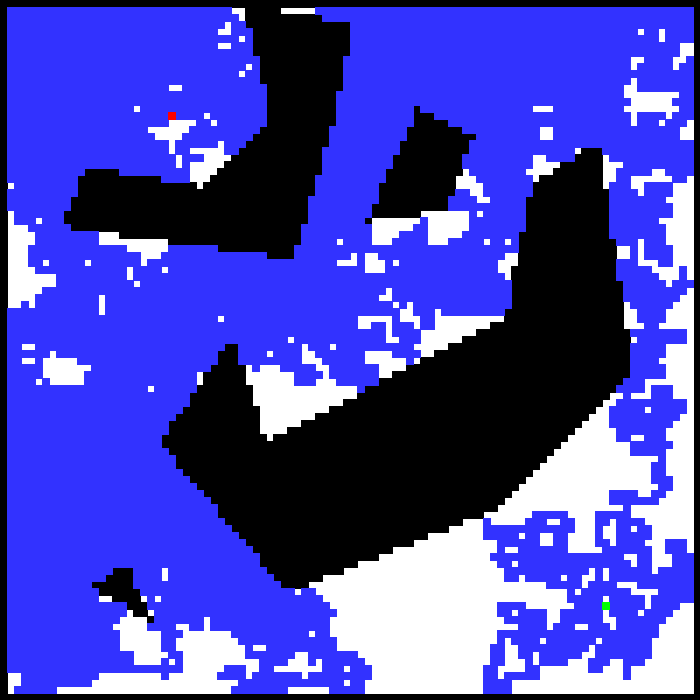} &
\includegraphics[width=4cm]{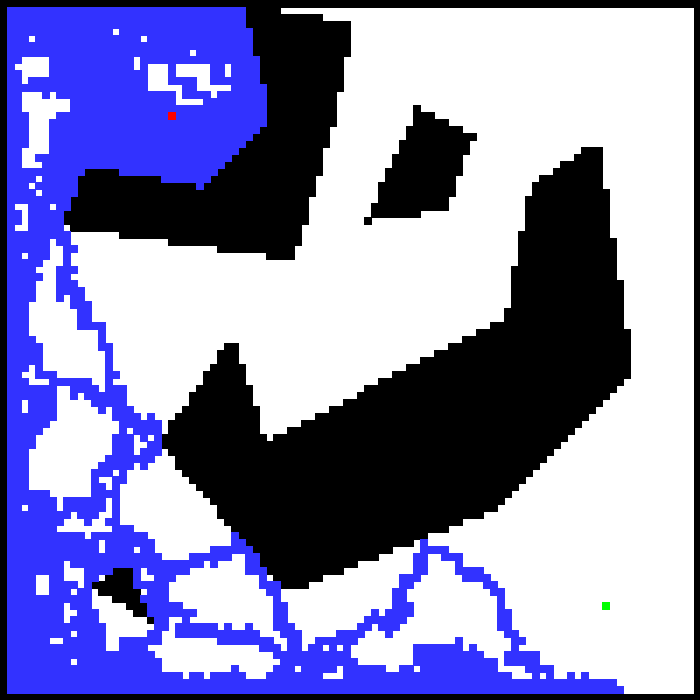} \\
a. & b. & c. \\
\includegraphics[width=4cm]{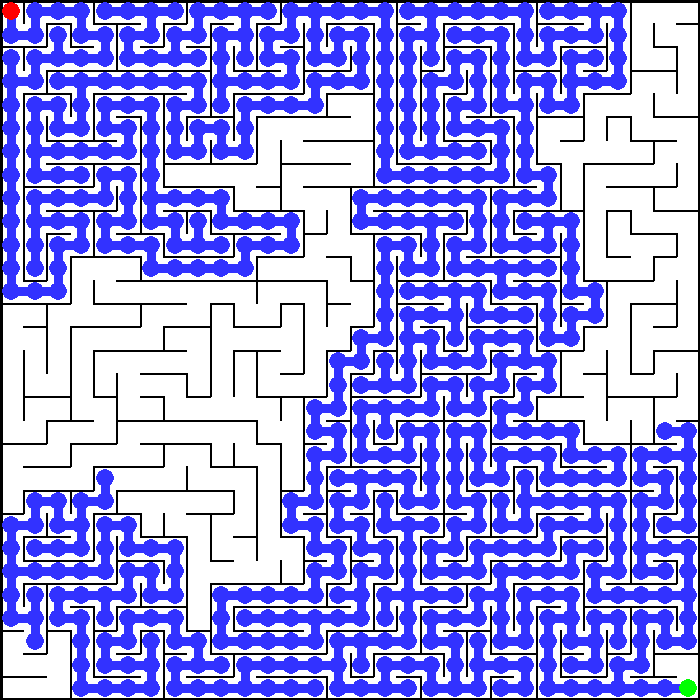} &
\includegraphics[width=4cm]{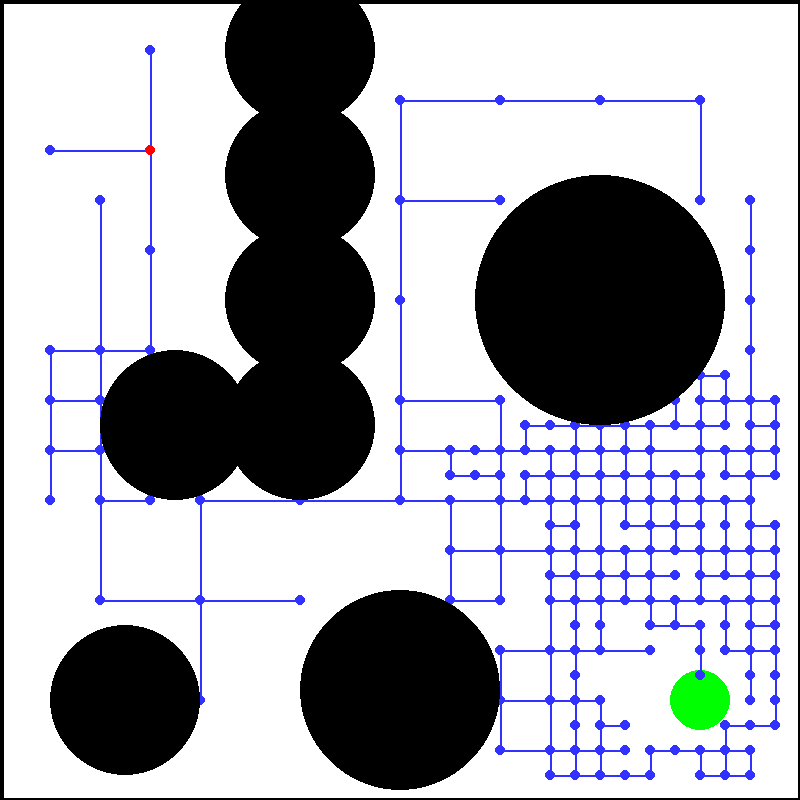} &
\includegraphics[width=4cm]{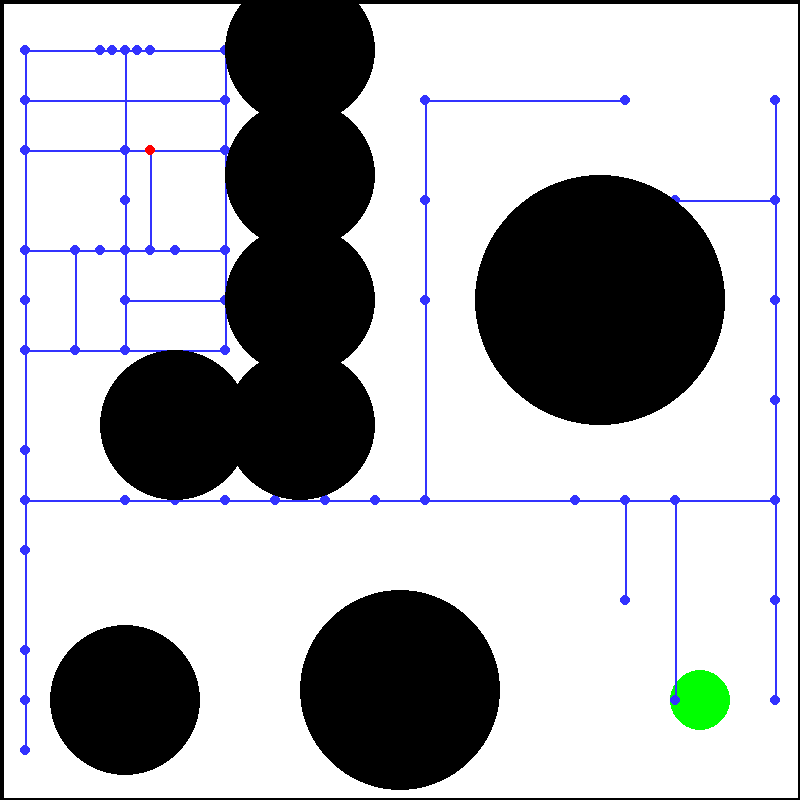} \\
d. & e. & f. 
\end{tabular}
\caption{\label{fig:sim} Typical executions of plans believed to be universal: (a) Digits of $\pi$ applied to search a 100x100 grid; 5548 of 7124 states were visited in 47785 total steps; all examples involve going from a red initial state in the upper left to a green goal state or region in the lower right.  (b) Same plan as in (a); 5220 of 7128 states were visited in 48291 steps.  (c) Applying Champernowne's number results in 11903745 steps; the progress after 50000 steps is shown. (d) Using $\pi$ digits, starting at position 1200000, to search a maze in 213675 steps. (e) Using $\pi$ digits starting at 12000000 to solve a continuous planning problem in 970 steps, using the proven method from Section \ref{sec:scale}.  (f) An alternative method, which seems more efficient, but is not yet proven to be complete; digits of $\pi$ starting with 3400000 were used.}
\end{figure}

We have experimentally investigated the application of universal plans to three kinds of problems:  1) grid world among obstacles, 2) classical mazes, and 3) planar path planning.  Figure \ref{fig:sim} illustrates some results.  All experiments use base-$4$ digits of Champernowne's number ($0123101112132021222330313233100101...$) or $\pi$ ($30210033312222020...$) for universal plans.  Figure \ref{fig:sim}(a)--(c) shows executions for a grid-searching robot.  For the first example, we ran 1000 trials, starting the $\pi$ digit sequence at position $100000(i-1)$ for the $i$th trial.  Large sequence starting indices were needed because trajectories generated from nearby offsets tend to quickly funnel together, decreasing diversity of executions. The average number of steps for the example in Figure \ref{fig:sim}(a) was 194,399, with a min and max of 10,549 and 997,595, respectively.  The average, min, and max for the example in Figure \ref{fig:sim}(b) was 79,448, 5,368, and 463,764, respectively.  Figure \ref{fig:sim}(c) shows the execution from alternatively applying Champernowne's number, which performs much worse than $\pi$.  The average, min, and max after 1000 trials was 107,870.540, 3,745, and 179,099,807, respectively.  We also ran 1000 trials on this example using pseudorandom actions, resulting in average, min, and max of 197,329, 10,495, and 1,220,295, with performance comparable to using $\pi$ digits.  

Figure~\ref{fig:sim}(d) shows a maze-searching study, again using $\pi$ digits with offsets and 1000 trials.  The resulting average, min, and max number of steps were 309,448, 28,460, and 2,419,499.   Figure~\ref{fig:sim}(e) shows a typical result from applying the method of Section \ref{sec:scale} with $w=1$ to a 2D planning problem among disc obstacles, to arrive in a disc goal region.  After 1000 trials using $\pi$ digits, the average, min, and max steps were 1,477.7, 38, and 65,809.  Figure~\ref{fig:sim}(f) illustrates a run of an alternative method, which yields average, min, and max steps 741.7, 36, and 4,166.  The method proceeds as follows.  The initial step size is chosen to be half the environment width.  In each iteration, two $\pi$ digits are used.  One selects the direction, and the other selects whether to halve or double the step size (keeping the initial size as a limit).  Digits 0, 1, and 2 cause it to be doubled, and 3 causes it to be halved, including a bias toward larger steps.  Although the method seems more efficient, we have not proved that it is universal (but conjecture that it is).

We tried various experiments with digits of other everyday irrationals than $\pi$, including $e$, $\ln 2$, $\sqrt 2$, and the golden ratio; their performance is generally comparable to using $\pi$ digits.  However, Champernowne is generally much worse, as may be expected because it is not strongly normal, as defined in~\cite{AraBaiBorBor13}.  Thus, it seems that the naturally occurring irrational numbers that are widely believed to be normal have better performance than the `hand-made' Champernowne number, but the latter is at least proved to be normal~\cite{BaiCra02}.  This results in a performance gap between using numbers proven to be normal versus believed to be normal!

\section{Conclusions}\label{sec:con}
We have thus shown that a planning algorithm may output a fixed plan, regardless of the input problem, and the algorithm will nevertheless be complete in the sense that the plan will bring the robot to the goal if it is possible.  It generally is not, however, able to decide whether it is feasible to reach the goal because no computation is performed over an input representation.  Our results are formulated for discrete grid planning and motion planning in a 2D freespace with regular boundary (positive reach), both of which were confirmed in simulation studies. However, with only slight modifications, the proofs apply to any finite dimensional state space, in both the discrete and continuous case. Several challenging open problems exist, including proving that the method in Figure \ref{fig:sim}(f) is a universal plan, developing improved universal motion plans, universal plans based on any rich (not necessarily normal) sequences, universal plans under differential constraints, universal plans under information feedback, discovering optimal plans in continuous spaces, and tight upper and lower bounds on time and space complexity.  We believe the insights gained in this work will be valuable for machine learning, especially in exploration phases of reinforcement learning, as well as motion planning, error detection and recovery, operating in GPS-denied environments, and system verification.

\subsubsection*{Acknowledgements.} We would like to thank our colleagues Ba\c{s}ak Sak\c{c}ak, Nicoletta Prencipe, Tuomas Orponen, and Vadim Weinstein for many fruitful discussions and suggestions during the process of writing this paper.

\bibliographystyle{plain}
\bibliography{main,more,pub}

\end{document}